\documentclass[letterpaper]{article} 
\usepackage{aaai24}  
\usepackage{times}  
\usepackage{helvet}  
\usepackage{courier}  
\usepackage[hyphens]{url}  
\usepackage{graphicx} 
\usepackage{natbib}  
\usepackage{caption} 
\usepackage{algorithm}
\usepackage{algpseudocode}

\usepackage{newfloat}
\usepackage{listings}
\DeclareCaptionStyle{ruled}{labelfont=normalfont,labelsep=colon,strut=off} 
\floatstyle{ruled}
\newfloat{listing}{tb}{lst}{}
\floatname{listing}{Listing}
\nocopyright  

\title{Probabilistic Neural Circuits}

\author {
    Pedro Zuidberg Dos Martires
}
\affiliations{
    Centre for Applied Autonomous Sensor Systems (AASS), Örebro University, Sweden
}

\usepackage{mathtools} 
\usepackage{booktabs} 
\usepackage{tikz}
\usetikzlibrary{decorations.pathreplacing,calligraphy,calc,hobby,intersections,through}
\usetikzlibrary{spn}

\usepackage[utf8]{inputenc} 
\usepackage[T1]{fontenc}    
\usepackage{url}            
\usepackage{booktabs}       
\usepackage{amsfonts}       
\usepackage{nicefrac}       

\usepackage{multirow}
\usepackage{courier}
\usepackage{listings, lstautogobble,amsfonts}
\usepackage{amsmath,amssymb,amsthm}
\usepackage{mdframed}
\usepackage{mathtools}
\usepackage{xspace}
\usepackage{xcolor}
\usepackage{caption}
\usepackage{multicol}
\usepackage{thmtools}
\usepackage{bm}
\usepackage{thm-restate}
\usepackage{todonotes}
\usepackage[inline]{enumitem}
\usepackage{soul}
\usepackage{physics}

\usepackage{microtype}
\usepackage{graphicx}
\usepackage{subfigure}
\usepackage{booktabs} 
\usepackage{cancel}

\pgfplotsset{compat=1.18}

\newcommand{\cf}{cf.\xspace}
\newcommand{\eg}{e.g.\xspace}
\newcommand{\ie}{i.e.\xspace}

\newtheorem{theorem}{Theorem}[section]
\newtheorem{definition}[theorem]{Definition}
\newtheorem{corollary}[theorem]{Corollary}
\newtheorem{proposition}[theorem]{Proposition}
\newtheorem{example}[theorem]{Example}

\newenvironment{talign}
{\align}
{\endalign}

\newcommand{\circuit}{\ensuremath{p}}
\newcommand{\inputs}{\ensuremath{\text{in}}}

\newcommand{\Xvars}{\ensuremath{\mathbf{X}}}
\newcommand{\xvars}{\ensuremath{\mathbf{x}}}
\newcommand{\Xvar}{\ensuremath{X}}
\newcommand{\xvar}{\ensuremath{x}}

\newcommand{\poset}{\ensuremath{\mathcal{O}}}

\newcommand{\weight}{\ensuremath{w}}
\newcommand{\nweight}{\ensuremath{\omega}}

\newcommand{\scope}{\ensuremath{\phi}}

\newcommand{\parents}{\ensuremath{{pa}}}
\newcommand{\ancestors}{\ensuremath{{an}}}

\newcommand{\component}{\ensuremath{{\kappa}}}

\newcommand{\midlinewidth}{1.0pt}

\newcommand{\middist}{24pt}

\newcommand{\smalldist}{20pt}

\newcommand{\halfdist}{4pt}

\makeatletter
\renewenvironment{proof}[1][\proofname]{\par
    \pushQED{\qed}%
    \normalfont\topsep0pt \partopsep0pt 
    \trivlist
    \item[\hskip\labelsep
                \itshape
                #1\@addpunct{.}]\ignorespaces
}{%
    \popQED\endtrivlist\@endpefalse
    \vskip 1ex  
}
\makeatother

\begin{document}

\maketitle

\begin{abstract}
    Probabilistic circuits (PCs) have gained prominence in recent years as a versatile framework for discussing probabilistic models that support tractable queries and are yet expressive enough to model complex probability distributions.
    Nevertheless, tractability comes at a cost: PCs are less expressive than neural networks.
    In this paper we introduce probabilistic neural circuits (PNCs), which strike a balance between PCs and neural nets in terms of tractability and expressive power. Theoretically, we show that PNCs can be interpreted as deep mixtures of Bayesian networks. Experimentally, we demonstrate that PNCs constitute powerful function approximators.
\end{abstract}

\section{Introduction}

In recent years probabilistic circuits (PCs) (also called sum-product networks)~\citep{darwiche2003differential,poon2011sum} have emerged as an assembly language to talk about tractable probabilistic models and inference therein~\citep{vergari2021compositional}. The core idea is quite simple: we start with a set of independent random variables and construct complex probability distribution by recursively adding and multiplying them together.
There are two common ways of interpreting PCs. Firstly, we can consider them to be hierarchical mixture models. Secondly, we look at them as neural nets consisting of sums, products, and atomic probability distributions.

Most of the recent advances in the field adhere to the second perspective: use an overparametrized probabilistic model and fit it to data using gradient based methods by leveraging discrete GPUs~\citep{peharz2019random,dang2021juice}. The computation units of such circuits are organized in a layered fashion. We give an example in Figure~\ref{fig:circuit}.

\begin{figure}[t]
    \centering
    \resizebox{0.85\columnwidth}{!}{
        \tikzset{point/.style={circle,inner sep=0pt,minimum size=3pt,fill=red}}

    \begin{tikzpicture}

        \varnode[line width=\midlinewidth]{v11}{$X_1$};
        \varnode[line width=\midlinewidth, right=\middist of v11]{v12}{$X_1$};
        \varnode[line width=\midlinewidth, right=\middist of v12]{v21}{$X_2$};
        \varnode[line width=\midlinewidth, right=\middist of v21]{v22}{$X_2$};

        \sumnode[line width=\midlinewidth, above=\smalldist of v11]{s11};
        \sumnode[line width=\midlinewidth, right=\halfdist of s11]{s12};
        \sumnode[line width=\midlinewidth, above=\smalldist of v12]{s1k};

        \sumnode[line width=\midlinewidth, above=\smalldist of v21]{s21};
        \sumnode[line width=\midlinewidth, right=\halfdist of s21]{s22};
        \sumnode[line width=\midlinewidth, above=\smalldist of v22]{s2k};

        \prodnode[line width=\midlinewidth, above=\smalldist of s1k]{p121};
        \prodnode[line width=\midlinewidth, right=\halfdist of p121]{p122};
        \prodnode[line width=\midlinewidth, above=\smalldist of s21]{p12k};

        \sumnode[line width=\midlinewidth, above=\smalldist of p121]{s121};
        \sumnode[line width=\midlinewidth, above=\smalldist of p122]{s122};
        \sumnode[line width=\midlinewidth, above=\smalldist of p12k]{s12k};

        \edge[line width=\midlinewidth,left] {s11, s12, s1k} {v11, v12};
        \edge[line width=\midlinewidth,left] {s21, s22, s2k} {v21, v22};

        \edge[line width=\midlinewidth,dashed] {p121} {s11, s21};
        \edge[line width=\midlinewidth,dashed] {p122} {s12, s22};
        \edge[line width=\midlinewidth,dashed] {p12k} {s1k, s2k};

        \edge[line width=\midlinewidth,left] {s121, s122, s12k} {p121, p122, p12k};

        \varnode[line width=\midlinewidth, right=\middist of v22]{v31}{$X_3$};
        \varnode[line width=\midlinewidth, right=\middist of v31]{v32}{$X_3$};
        \varnode[line width=\midlinewidth, right=\middist of v32]{v41}{$X_4$};
        \varnode[line width=\midlinewidth, right=\middist of v41]{v42}{$X_4$};

        \sumnode[line width=\midlinewidth, above=\smalldist of v31]{s31};
        \sumnode[line width=\midlinewidth, right=\halfdist of s31]{s32};
        \sumnode[line width=\midlinewidth, above=\smalldist of v32]{s3k};

        \sumnode[line width=\midlinewidth, above=\smalldist of v41]{s41};
        \sumnode[line width=\midlinewidth, right=\halfdist of s41]{s42};
        \sumnode[line width=\midlinewidth, above=\smalldist of v42]{s4k};

        \prodnode[line width=\midlinewidth, above=\smalldist of s3k]{p341};
        \prodnode[line width=\midlinewidth, right=\halfdist of p341]{p342};
        \prodnode[line width=\midlinewidth, above=\smalldist of s41]{p34k};

        \sumnode[line width=\midlinewidth, above=\smalldist of p341]{s341};
        \sumnode[line width=\midlinewidth, above=\smalldist of p342]{s342};
        \sumnode[line width=\midlinewidth, above=\smalldist of p34k]{s34k};

        \edge[line width=\midlinewidth,left] {s31, s32, s3k} {v31, v32};
        \edge[line width=\midlinewidth,left] {s41, s42, s4k} {v41, v42};

        \edge[line width=\midlinewidth,dashed] {p341} {s31, s41};
        \edge[line width=\midlinewidth,dashed] {p342} {s32, s42};
        \edge[line width=\midlinewidth,dashed] {p34k} {s3k, s4k};

        \edge[line width=\midlinewidth,left] {s341, s342, s34k} {p341, p342, p34k};

        \prodnode[line width=\midlinewidth, above=90pt of s2k]{p12341};
        \prodnode[line width=\midlinewidth, right=\halfdist of p12341]{p12342};
        \prodnode[line width=\midlinewidth, above=90pt of s31]{p1234k};

        \sumnode[line width=\midlinewidth, above=\smalldist of p12342]{s12341};

        \edge[line width=\midlinewidth,dashed] {p12341} {s121, s341};
        \edge[line width=\midlinewidth,dashed] {p12342} {s122, s342};
        \edge[line width=\midlinewidth,dashed] {p1234k} {s12k, s34k};

        \edge[line width=\midlinewidth,left] {s12341} {p12341, p12342, p1234k};

        \draw [decorate, decoration = {brace, mirror, amplitude=10pt}, ultra thick] (6.2,4.7) --  (6.2,6.8) node[pos=0.5,right=10pt]{ \begin{tabular}{c}
                root layer \\(Layer 3)
            \end{tabular} };
        \draw [decorate, decoration = {brace, mirror, amplitude=10pt}, ultra thick] (9.1,2.2) --  (9.1,4.3) node[pos=0.5,right=10pt]{\begin{tabular}{c}
                sum-product layer \\(Layer 2)
            \end{tabular} };
        \draw [decorate, decoration = {brace, mirror, amplitude=10pt}, ultra thick] (10.5,-0.4) --  (10.5,1.7) node[pos=0.5,right=10pt]{\begin{tabular}{c}
                leaf layer \\(Layer 1)
            \end{tabular} };

        \node[draw=red, circle,dashed,ultra thick, minimum width=0.7cm] at (p121.center) (c1) {};
        \node[draw=red, circle,dashed,ultra thick, minimum width=0.7cm] at (p122.center) (c2) {};
        \node[draw=red, circle,dashed,ultra thick, minimum width=0.7cm] at (p12k.center) (c3) {};

        \node[draw=red, circle,dashed, ultra thick, minimum width=0.7cm, below=1.2 of p12342, inner sep=0pt]  (mid) {};

        \path (c1) edge[-, red,dashed, ultra thick, in=160, out=35] (mid);
        \path (c2) edge[-, red,dashed, ultra thick, in=175, out=45] (mid);
        \path (c3) edge[-, red,dashed, ultra thick, in=220, out=-20] (mid);

        \path (mid) edge[->, red,dashed, ultra thick, in=190, out=25] (s341);
        \path (mid) edge[->, red,dashed, ultra thick, in=220, out=10] (s342);
        \path (mid) edge[->, red,dashed, ultra thick, in=230, out=-10] (s34k);

    \end{tikzpicture}
    }
    \caption{Layered probabilistic circuit following the construction of~\citep{shih2021hyperspns}. Data (modeled as random variables) is first fed into the leaf layer at the bottom. The output of the leaf layer is a mixture of distributions produced by the sum units. In the sum-product layer (Layer 2)  mixtures of random variables are combined by taking pairwise products, these are then again mixed using sum units. Finally, the root layer (at the top) gives us the joint probability distribution. The red edges indicate functional dependencies not present in traditional probabilistic circuits but present in probabilistic neural circuits.}
    \label{fig:circuit}
\end{figure}
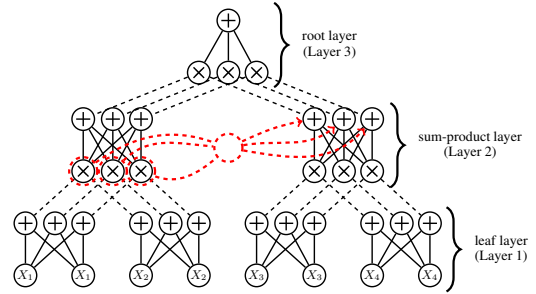

A major advantage of PCs is their ability to answer certain queries in polynomial time -- given that adequate restrictions are imposed on a circuit's structure~\citep{vergari2021compositional}. An example of such a tractable query would be the computation of conditional probabilities for so-called \textit{smooth and decomposable} PCs~\citep{darwiche2001decomposable,darwiche2003differential}.

This tractability, however, comes at a hefty price: the properties imposed on PCs in order to ensure polynomial time queries limit their expressive power~\citep{martens2014expressive,sharir2018sum,zhang2021probabilistic}.
This is in contrast to general neural networks and even sum-product networks with fewer structural constraints~\citep{delalleau2011shallow,kileel2019expressive}. \citet{martens2014expressive} have shown that decomposability is in fact a necessary condition for tractable marginal inference.

Nevertheless, using the concepts of \textit{conditional smoothness} and \textit{conditional decomposability}~\citep{sharir2018sum}, we study in this paper the space of models that lie in between probabilistic circuits and neural networks.
Concretely, we make the following contributions:
\begin{enumerate}
    \item We introduce conditional probabilistic circuits, from which we construct  probabilistic neural circuits (PNCs), which we interpret as deep mixtures of Bayesian nets.
    \item  We provide a prescription to construct layered PNCs.
    \item We provide an implementation of layered PNCs and experimentally study their expressive power.
\end{enumerate}

Our work is influenced by that of \citet{sharir2018sum}. We discuss the relationship to their approach (dubbed sum-product-quotient networks) in Section~\ref{sec:related}.

\section{Preliminaries}
\label{sec:prelim}

In the remainder of the paper we will denote random variables by uppercase $\Xvar$'s, the corresponding values are lowercase $\xvar$'s. Sets of random variables and their corresponding values are typed in boldface: $\Xvars$ and $\xvars$, respectively.
The definitions and notions we introduce in this section are loosely based on the work of~\citet{vergari2021compositional}.

\begin{definition}[Probabilistic Circuit]
    \label{def:circuit}
    A probabilistic circuit over random variables $\Xvars$ is a parametrized computational graph encoding a probability density function $\circuit(\Xvars{=}\xvars)$.
    The circuit consists of three kinds of computational units:
    \textit{leaf}, \textit{product}, and \textit{sum}.
    Each product or sum unit receives inputs from a set of input units denoted by $\inputs(n)$.
    Each unit $k$ encodes a function $\circuit_{k}(\cdot)$ as follows:
    \begin{align*}
        {\circuit}_k(\xvars_n)=
        \begin{cases}
            f_k({\xvars_n})                                             & \text{if $k$ leaf unit}    \\
            \circuit_{k_l}(\xvars_{n_l})  \circuit_{k_r} (\xvars_{n_r}) & \text{if $k$ product unit} \\
            \sum_{j\in\inputs(k)} \weight_{kj} \circuit_j(\xvars_n)     & \text{if $k$ sum unit}
        \end{cases}
    \end{align*}
    where $f_k(\xvars_n)$  denotes a parametrized probability distribution having as support the sample space of the random variables in $\Xvars_n$.
\end{definition}

\begin{definition}[Scope]
    \label{def:scope:cond}
    The scope of a unit $k$, denoted by $\scope(k)$, is the set of random variables $\Xvars_n$ for which the function $\circuit_k (\cdot)$ encodes a probability distribution.
\end{definition}

Two important properties that are usually imposed on probabilistic circuits are smoothness and decomposability as they allow for tractable queries, \eg computing marginals~\citep{darwiche2001decomposable,darwiche2003differential}.
\begin{definition}[Smoothness]
    A circuit is smooth if for every sum unit $k$ its inputs encode distributions over the same random variables:
    $\forall j_1, j_2 {\in} \inputs(k)$ it holds that $\scope(j_1){=}\scope(j_2) $.
\end{definition}

\begin{definition}[Decomposability]
    A circuit is decomposable if the inputs of every product unit
    $k$ encode distributions over disjoint sets of random variables:
    $\scope (k_l) \cap \scope (k_r) = \emptyset$ with $\{k_l, k_r\}= \inputs (k)$.
\end{definition}

\begin{definition}[Valid Probabilsitic Circuit]
    We call a probabilistic circuit valid if for every unit $k$ we have that $p_k(\xvars_n){\geq} 0$ and $\int p_k(\xvars_n) \differential \xvars_n{=}1$.\footnote{Note that our notion of validity is slightly stricter than in its original definition, \cf \citep{poon2011sum}}
\end{definition}

As discussed by \citet{peharz2015theoretical}, probabilistic circuits are valid if they are smooth, decomposable, and that for the weights in the sum units we have $\weight_{kj} \in \mathbb{R}^+$ and $\sum_{j\in\inputs(k)} \weight_{kj} = 1$ for every $k$. Note that the notation in Definition~\ref{def:circuit} already suggests that the circuit is smooth as the inputs to the sum units are functions over the same set of variables $\Xvars_n$.

Furthermore, we can assume, without loss of generality, that sum and product units occur in an alternating fashion in the circuit~\citep{peharz2020einsum}.
This observation naturally leads to a layer-wise construction of probabilistic circuits where consecutive layers alternate between sum and product layers.
Such layered probabilistic circuits~\citep{peharz2019random} have the advantage that the computations within a layer can be trivially parallelized.
We can further abstract the layers in a circuit by fusing together alternating sums and products into a single sum-product layer~\citep{peharz2020einsum}.

In Figure~\ref{fig:circuit} we give a graphical representation of a layered circuit. Layers consist of  blocks of computational units that are processed sequentially in a bottom-up fashion. Layers are themselves constituted of so-called partitions. The circuit in Figure~\ref{fig:circuit} has four partitions in the leaf layer, two in the sum-product layer, and a single partition in the root layer. By construction, partitions in the same layer have disjoint scopes. Moreover, partitions are further subdivided into input components and output components, which constitute the elemental computing units. The circuit in Figure~\ref{fig:circuit} has, except at the very bottom and top, three such input and output components in each partition.

We can uniquely identify each computational unit (or component) in the circuit by specifying the layer, the partition, whether it is an input or an output, and its position within a partition. Counting layers from bottom to top, and partitions and units from left to right. Each component can be identified using 4 indices: $\component_{l,p,i,c}$. The first index $l$ identifies the layer, the second $p$ the partition, the third $i \in \{1,2 \}$ whether it is an input or output, and the fourth $c$ the horizontal position within a partition.
For instance, the symbol $\component_{2223}$ corresponds to  the upper-right unit in the sum-product layer.

\section{Conditional Probabilistic Circuits}
\label{sec:cpc}

We will first introduce the notion of posets (partially ordered sets) of random variables (Section~\ref{sec:poset}).
This will allow us to generalize probabilistic circuits to conditional probabilistic circuits, which we interpret as deep mixtures of Bayesian networks (Section ~\ref{sec:dmbn}).
We then introduce probabilistic neural circuits and their tractable queries (Section~\ref{sec:pnc}).

\subsection{Partially Ordered Random Variables}
\label{sec:poset}

Consider a set of random variables $\Xvars {=}\{\Xvar_1,\dots, \Xvar_N \}$ on which we impose the \textit{parents} relationship $\parents(\cdot)$.
The parents relationship induces a directed acyclic graph on the random variables $\Xvars$, where the nodes are the random variables themselves and an edge is present between $\Xvar_i$ and $\Xvar_j$ if $\Xvar_i {\in} \parents(\Xvar_j)$.
This gives us a partial ordering of the variables $\Xvars$.
We also define the ancestor relationship $\ancestors(\cdot)$ to be the transitive closure of $\parents(\cdot)$. That is, the ancestors of a random variable are its direct parents and recursively their parents.
Furthermore, we denote the poset for the random variables $\Xvars$ by $\mathcal{O}(\Xvars)$. We say that the relationship $\Xvars_i {\sqsubset} \Xvars_j$ between two sets holds if
$
    \forall \Xvar_r {\in} \Xvars_i, \Xvar_q {\in} \Xvars_j : \Xvar_q {\notin}  \ancestors(\Xvars_r)
$.
We also define the relation $\Xvar_i\sqsubset \Xvar_j$ on random variables as $\{\Xvar_i \}{\sqsubset} \{\Xvar_j \}$.

\begin{example}[Bayesian Network]
    Partially ordered random variables induce a factorization of a joint probability distribution. A prominent example of such a factorization are Bayesian networks (\cf Figure ~\ref{fig:bn}) where we have:
    \begin{talign}
        \label{eq:bn}
        p(\Xvars{=}\xvars) = \prod_{n:\Xvar_n \in \Xvars} p_n(X_n{=}x_n \mid \Xvars_{\parents(n)}{=}\xvars_{\parents(n)}).
    \end{talign}
\end{example}

\begin{figure}

    \begin{minipage}{0.50\linewidth}
        \begin{align*}
            X_1 \sqsubset X_2 &  & X_1 \sqsubset X_4 \\
            X_1 \sqsubset X_3 &  & X_3 \sqsubset X_1 \\
            X_3 \sqsubset X_2 &  & X_3 \sqsubset X_4 \\
            X_2 \sqsubset X_3 &  & X_2 \sqsubset X_4 \\
        \end{align*}

    \end{minipage}%
    \begin{minipage}{0.49\linewidth}
        \resizebox{0.9\linewidth}{!}{

            \centering

            \begin{tikzpicture}
                \node[draw, circle,ultra thick] (x3) {$X_3$};
                \node[draw, circle, below=of x3,ultra thick] (x2) {$X_2$};
                \node[draw, circle,  right=2cm of $(x3)!0.5!(x2)$3,ultra thick] (x4) {$X_4$};
                \node[draw, circle, left=of x2,ultra thick] (x1) {$X_1$};

                \draw[->, ultra thick] (x1) -- (x2);
                \draw[->,ultra thick] (x2) -- (x4);
                \draw[->, ultra thick] (x3) -- (x4);
            \end{tikzpicture}
        }
    \end{minipage}
    \caption{Right: Bayesian network. Left: partial order relations that hold.}
    \label{fig:bn}
\end{figure}
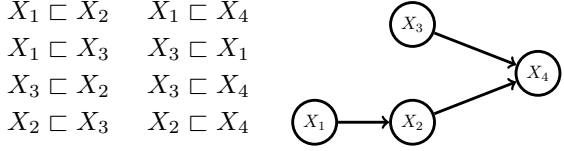

In the example above we denote $\parents(X_n)$ by $\Xvars_{\parents(n)}$. This will allow us to omit the random variable $\Xvar_n$ when writing down a probability distribution and only use the instantiation $\xvar_n$ instead.

\subsection{Deep Mixtures of Bayesian Networks}
\label{sec:dmbn}

\begin{definition}[Conditinal Probabilsitic Circuit (CPC)]
    \label{def:con:circuit}
    A CPC $\circuit_k$ over a poset $\poset(\Xvars)$ is a parametrized computational graph encoding a probability density function $\circuit(\Xvars{=}\xvars)$.
    The  CPC consists of
    \textit{leaf}, \textit{product}, and \textit{sum} units.
    Each unit $k$ encodes a function $\circuit_{k}$ as follows:
    \begin{align}
        \label{eq:con:circuit}
         & {\circuit}_k(\xvars_n \mid  \xvars_{\parents(n)})
        \\
         & =
        \begin{cases}
            f_k({\xvars_n}\mid   \xvars_{\parents(n)})                                          & \text{if leaf}    \\
            \circuit_{k_l}(\xvars_{n_l}\mid   \xvars_{\parents(n_l)})
            \circuit_{k_r} (\xvars_{n_r}\mid \xvars_{\parents(n_r)})                            & \text{if product} \\
            \sum_{j\in\inputs(k)} \weight_{kj} \circuit_j(\xvars_n  \mid  \xvars_{\parents(n)}) & \text{if sum}
        \end{cases}
        \nonumber
    \end{align}
    where $f_k(\xvars_n)$ denotes a parametrized probability distribution having as support the sample space of the random variables in $\Xvars_n$.
\end{definition}

\begin{definition}[Scope (CPC)]
    \label{def:scope}
    The scope  $\scope(k)$ of a unit $k$ encoding the probability distribution $\circuit_k(\xvars_n\mid \xvars_{\parents(n)} )$ is the set of random variables $\Xvars_n$.
\end{definition}

\begin{corollary} A conditional probabilistic circuit over an unordered set of random variables is a (non-conditional) probability circuit.
\end{corollary}
\begin{proof}
    Having no order means that $\parents(\Xvars)=\emptyset$ for every $\Xvar \in \Xvars$. This then means that the conditioning sets in Definition~\ref{def:con:circuit} are all empty and that we recover a (non-conditional) probabilistic circuit.
\end{proof}

We can now also introduce the notions of \textit{conditional smoothness} and \textit{conditional decomposability}:
\begin{definition}[Conditional Smoothness]
    A CPC is conditionally smooth if for every sum unit $k$ it holds that
    $\forall j_1, j_2 {\in} \inputs(k):  \scope(j_1){=}\scope(j_2) $
\end{definition}

\begin{definition}[Conditional Decomposability]
    A CPC is conditionally decomposable if for every product unit
    $k$ it holds that $\scope (k_l) \cap \scope (k_r) = \emptyset$ with $\{k_l, k_r\}= \inputs (k)$.
\end{definition}

\begin{corollary}
    A conditionally smooth and conditionally decomposable CPC is smooth and decomposable if its random variables $\Xvars$ are unordered.
\end{corollary}

\begin{definition}[Valid CPC]
    \label{def:valid_cpc}
    We call a CPC valid if for every unit $k$ we have that $p_k(\xvars_n {\mid} \xvars_{\parents(n)}){\geq} 0$ and $\int p_k(\xvars_n {\mid} \xvars_{\parents(n)}) \differential \xvars_n{=}1$.
\end{definition}

\begin{theorem}[Validity for CPCs]
    A CPC is valid if it is conditionally smooth, conditionally decomposable, and for every sum unit $n$ it holds that $\weight_{nm}{\geq} 0$ and $\sum_{m\in\inputs(n)} \weight_{nm}{=}1$.
\end{theorem}

\begin{proof}
    We start by rewriting the alternating sums and products of a CPC in its flat representation using the fact that products distribute over summations:
    \begin{align}
        \circuit_k(\xvars_n)
         & =  \sum_{\tau \in \mathcal{T}} \weight_\tau  \prod_{\rho \in \tau} \rho (\xvars_\rho)
        \label{eq:tree_decomp_PCP}
    \end{align}
    Here, $\mathcal{T}$ denotes the set of all products of leaf distributions in the CPC and $\rho \in \tau$ denotes a factor in one of these products
    (we refer to \citep{zhao2016unified} for a more detailed account).

    Invoking decomposability of the product units we have that each random variable $X_n$ only picks up a single factor $\rho (\xvar_\rho)$, which means that we can identify each $\rho (\xvar_\rho)$ with a specific $f_k(\xvar_n\mid \xvars_{\parents(n)} )$. This lets us rewrite Equation~\ref{eq:tree_decomp_PCP} as:
    \begin{align}
        \circuit_k(\xvars_n)
         & =  \sum_{\tau \in \mathcal{T}} \weight_\tau  \prod_{k: f_k \in \tau} f_k (\xvars_n \mid \xvars_{\parents(n)})
        \label{eq:f_tree_decomp_PCP}
    \end{align}
    Given that the $f_k$ are by definition (conditional) probability distributions their product forms a joint probability distribution as well. Next, we exploit smoothness, which states that the inputs to sum units have identical scope. This means that all the terms in the flat representation of $\circuit (\Xvars_n)$ mention the identical set of random variables and each term in the flat representation forms indeed a joint probability over the random variables $\Xvars_n$.

    Lastly, \citet{peharz2015theoretical} have shown that having normalized weights in the sum units of a circuit results in normalized weights $\weight_\tau$ in the flat representation. This lets us conclude that the circuit $p_k(\xvars_n)$ is a probability distribution. Note that we did not make any reference in our reasoning to any specific unit in the circuit. This means that our argument holds for all units in a conditionally smooth and conditionally decomposable circuit with normalized weights, which means in turn that such a circuit is valid.
\end{proof}

In light of Equation~\ref{eq:f_tree_decomp_PCP} and comparing it to Equation~\ref{eq:bn}, we can interpret CPCs as deep (or hierarchical) mixtures of Bayesian networks. This is analogous to interpreting probabilistic circuits as deep mixtures of fully factorized distributions~\citep{poon2011sum}.

\subsection{PNCs and Their Tractable Queries}
\label{sec:pnc}

The computational efficiency of probabilistic circuits stems from the fact the circuits evaluations are broken down into sub-evaluations, which are then cached and reused.
Inspecting, however, the functional form of the sum units in a CPC (\cf Equation~\ref{eq:con:circuit}), this is not the case:
each term in the sum over $j$ requires a separate conditional probability for each instantiation of the variables $\Xvars_{\parents(n)}$.
This means that we would need (assuming binary random variables) $2^{|\Xvars_{\parents(n)}|}$ functions to encode the conditional probabilities.
We alleviate this issue as follows. First, we rewrite the functional form of the sum units using Bayes' rule:
\begin{align}
    \sum_{j\in\inputs(k)}
    \weight_{kj}
    \frac{
        \circuit_j(  \xvars_{\parents(n)} \mid \xvars_n  )
    }{
        \circuit_j(  \xvars_{\parents(n)  })
    }
    \circuit_j(  \xvars_n)
\end{align}
Second, we make the following approximation:
\begin{align}
    \weight_{kj}
    \frac{
        \circuit_j(  \xvars_{\parents(n)} \mid \xvars_n  )
    }{
        \circuit_j(  \xvars_{\parents(n)  })
    }
    \approx
    \nweight_{kj}(\xvars_{\ancestors(n)}),
\end{align}
where $\nweight_{kj}(\cdot)$ is a neural network depending on the set of ancestors $\ancestors(\xvar_n)$. This now allows us to formally introduce probabilistic neural circuits.
\begin{definition}[Probabilistic Neural Circuit (PNC)]
    \label{def:pnc}
    A PNC is a conditionally smooth and conditionally decomposable  CPC where sum units take the following functional form:
    \begin{talign}
        {\circuit}_k(\xvars_n \mid  \xvars_{\parents(n)})
        =
        \sum_{j\in\inputs(k)} \nweight_{kj} ( \xvars_{\ancestors(n)}) \circuit_j(\xvars_n),
        \label{eq:pnc}
    \end{talign}
    with $\nweight_{kj}: \Omega( \ancestors(\xvar_n)) \rightarrow [0,1 ]$ being a neural network mapping from the sample space of the random variables $\ancestors(\xvar_n)$ to a real  between zero and one, and where it holds that $\sum_{j\in\inputs(k)} \nweight_{kj} = 1$.
\end{definition}
It can easily be seen that PNCs, already encode valid circuits (\cf. Definition~\ref{def:valid_cpc}) by construction.
Intuitively, we interpret PNCs as neural approximations of CPCs.
Note that it is this approximation that makes PNCs tractable: PNCs only need a single circuits for each $j$ while CPCs need a circuit for every $j$ and every instantiation of $\xvars_{\parents(n)}$.

\begin{proposition}
    Probabilistic circuits are PNCs.
\end{proposition}
\begin{proof}
    If the values $\nweight(\xvars_{\ancestors(n)})$ do not depend on $\xvars_{\ancestors(n)}$ we have $|\inputs(k)|$ constants that sum up to $1$. This means that the weights in the sum units do not depend on the conditioning sets from the antecedent product layer, and we can omit any conditioning sets. The definition of a PNC in this case is then equivalent to the definition of a probabilistic circuit.
\end{proof}

Given the definition of PNCs we can now determine tractable queries that we can perform.

\begin{proposition}[Density Evaluation]
    Given a probabilistic circuit $p_k$ over the random variables $\Xvars$. We can evaluate the circuit  at the instantiation $\xvars$ in linear time with respect to the size of the circuit.
\end{proposition}
\begin{proof}
    This follows simply from the fact that a circuit is a (non-recurrent) computation graph and that we can simply evaluate it by computing input units before output units.
\end{proof}

\begin{proposition}[Ordered Marginals]
    \label{prop:ord_marg}
    Consider a PNC $\circuit_k(\xvars_m, \xvars_e)$  we can then compute the marginal $\circuit_k(\xvars_e)$ in polynomial time if $\Xvars_e {\sqsubset} \Xvars_m$.
\end{proposition}
\begin{proof}
    We start by writing out explicitly the single elements in the set $\xvars_m$:
    \begin{align}
        \circuit_k(\xvars_m, \xvars_e) = \circuit_k(\xvar_\mu, \dots, \xvar_1, \xvars_e)
    \end{align}
    where the order of $\{ \xvar_\mu, \dots, \xvar_1\}= \xvars_m$ is arbitrary but respects the partial order $\poset(\Xvars_m \cup \Xvars_e )$.
    Given that the circuit is conditionally smooth and conditionally decomposable, we know that it encodes a proper probability distribution over its variables.
    We can hence obtain the marginal $\circuit(\xvars_e)$ by integrating over the possible values $\xvars_m$:
    \begin{align}
        \circuit_k(\xvars_e)
        =
        \int \cdots \int p_k(\xvar_\mu, \dots, \xvar_1,\xvars_e)  \differential  \xvar_\mu \cdots \differential \xvar_1
    \end{align}
    As $\xvar_\mu$ does not appear in any of the conditioning sets and as any product unit $q$ is decomposable we can simply push the summation to the input unit $r$ of $q$ for which we have  $\Xvar_\mu \in \scope (r)$.
    For summation units we exploit the linearity of the integral and distribute the integral over the terms in the sum.
    Performing this recursively brings the integral to the leaves where we have  $ \int f_i(x_\mu \mid x_{\parents(\mu)} ) \differential \xvar_\mu =1$.

    Up to this point the marginalization in CPC is identical to marginalization in probabilistic circuits. Contrary, to probabilistic circuits, however, we now need to propagate back up the result of the marginalization. Assuming, without loss of generality that leaf units feed into sum units, we then have
    \begin{align*}
         & \sum_{g\in\inputs(h)} \nweight_{hg} (\xvars_{\ancestors( \mu)})
        \int \circuit_g (\xvar_\mu  ) \differential \xvar_\mu
        {=} \sum_{g\in\inputs(h)} \nweight_{hg} (\xvars_{\ancestors( \mu)})
        {=}1
    \end{align*}
    The next product node that we encounter on our way up through the circuit is of the form:
    \begin{align}
         & \circuit_{l}(\xvars_{\mu-1}\mid   \xvars_{\parents(\mu-1)})
        \int  \circuit_{r} (\xvars_{\mu}\mid \xvars_{\parents(\mu)}) \differential \xvars_\mu
        \nonumber
        \\
         & =\circuit_{l}(\xvars_{\mu-1}\mid  \xvars_{\parents(\mu-1)})
        \nonumber
    \end{align}
    At this point we have integrated out the variable $\xvar_\mu$ from the circuit by traversing a number of units linear in the size of the circuit (assuming proper caching~\citep{vergari2021compositional}). Repeating this procedure for the remaining set of ordered random variables $\{\Xvar_{\mu-1},\dots,\Xvar_1 \}$ gives us the distribution $\circuit_k(\xvars_e)$ in polynomial time.
\end{proof}
The proof follows a similar reasoning to the case of probabilistic circuits. The delicate point was to show that in the product unit one of the factors drops out. This is important as $\Xvar_{\parents(\mu)}$ might include $\Xvar_{\mu-1}$. Retaining such a dependency would prevent us from performing tractable ordered marginalization.

\begin{corollary}[Ordered Conditionals]
    Assuming a PNC $\circuit_k (\xvars_m, \xvars_o, \xvars_e)$ where $\Xvars_e \sqsubset \Xvars_o \sqsubset \Xvars_m$ holds lets us compute the conditional $p(\xvars_o \mid \xvars_e)$ in polynomial time.
\end{corollary}
\begin{proof}
    We first apply the definition of the conditional probability:
    $p(\xvars_o \mid \xvars_e)=\nicefrac{p(\xvars_o , \xvars_e)}{p( \xvars_e)}$.
    Using the law of total probability we rewrite the ratio as
    $$
        p(\xvars_o \mid \xvars_e)
        =
        \frac
        {
            \int p(\xvars_m , \xvars_o, \xvars_e) \differential \xvars_m
        }
        {
            \int p(\xvars_m , \xvars_o, \xvars_e) \differential \xvars_m \differential \xvars_o
        }
    $$
    and Proposition~\ref{prop:ord_marg} tells us that we can perform both marginalizations in polytime.
\end{proof}



\section{Layered Probabilistic Neural Circuits}

While Equation~\ref{eq:pnc} provides a generic functional expression to compute the sum units in PNCs, it is not clear how to construct a PNC in the first place. That is, how do we link up the individual computation units such that they form a (valid) CPC.
For (non-conditional) probabilistic circuits potent structure learning algorithms have been developed in recent years  (\eg  hidden Chow-Liu trees~\citep{liu2021tractable} or random probabilistic circuits~\citep{di2021random}). It is not entirely clear how to adapt these to the setting of conditional probabilistic circuits. For this reason we study problems where, informed by the structure of the data itself, a structure for a PNC can be constructed. Concretely, we will study PNC structures tailored towards image data: features, \ie pixels, that are close to each other should also be close to each other in the circuit -- a fact already exploited by~\citet{poon2011sum}.

The structure we propose in this paper is rather simple and inspired by simple feed-forward neural networks and also layered (non-conditional) probabilistic circuits~\citep{peharz2020einsum}. More concretely,
the computation units in the current layer only depend on computation units in the previous layer. Furthermore, we wish the units within each layer to be computable in parallel (given the previous layer). To construct such a probabilistic neural circuit we take the circuit structure introduced by \citet{shih2021hyperspns} as a backbone and add additional edges to the computation graph in order to obtain a PNC from a probabilistic circuit.
For ease of exposition we detail our approach using one-dimensional data instead of two-dimensional data.

\subsection{Structure for One-Dimensional Data}

In order to study PNC structures, we introduce the concept of a partition graph, which is a hypergraph of a probabilistic circuit with the partitions of a circuit being the nodes and edges encoding the sub-partitions\footnote{In the circuit structure introduced by ~\citet{shih2021hyperspns}, all the product nodes decompose in the same fashion, \ie for product nodes with the same scope it is the sames variables that come from the left and right inputs, respectively. This is also called structured decomposabiltiy~\citep{darwiche2011sdd}. Partition trees are related to the concept of \textit{variable trees} in \textit{(probabilistic) sentential decision diagrams}~\citep{darwiche2011sdd,kisa2014probabilistic}. However, nodes in a variable tree do not constitute abstractions of a specific group of sum and product units and are a more general concept. }.  We give an example of a partition tree in Figure~\ref{fig:partition_graph}.

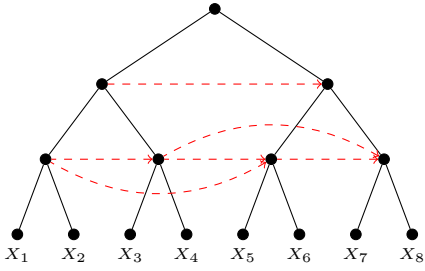
\begin{figure}[t]
    \begin{minipage}[c]{\linewidth}
        \centering

        \begin{tikzpicture}[
                dot/.style = {circle, draw, fill, minimum size=4pt, inner sep=0pt},
                level distance=1cm,
                level 1/.style={sibling distance=3cm},
                level 2/.style={sibling distance=1.5cm},
                level 3/.style={sibling distance=0.75cm},
                scale=1
            ]

            \node[dot] (00) {}
            child {node[dot] (10) {}
                    child {node[dot] (20) {}
                            child {node[dot] (30) {}}
                            child {node[dot] (31) {}}
                        }
                    child {node[dot] (21) {}
                            child {node[dot] (32) {}}
                            child {node[dot] (33) {}}
                        }
                }
            child {node[dot] (11) {}
                    child {node[dot] (22) {}
                            child {node[dot] (34) {}}
                            child {node[dot] (35) {}}
                        }
                    child {node[dot] (23) {}
                            child {node[dot] (36) {}}
                            child {node[dot] (37) {}}
                        }
                };

            \node[below= 0.1 of 30, inner sep=0pt] (X1) {\tiny{$X_1$}};
            \node[below= 0.1 of 31, inner sep=0pt] (x2) {\tiny{$X_2$}};
            \node[below= 0.1 of 32, inner sep=0pt] (x3) {\tiny{$X_3$}};
            \node[below= 0.1 of 33, inner sep=0pt] (x4) {\tiny{$X_4$}};
            \node[below= 0.1 of 34, inner sep=0pt] (X5) {\tiny{$X_5$}};
            \node[below= 0.1 of 35, inner sep=0pt] (X6) {\tiny{$X_6$}};
            \node[below= 0.1 of 36, inner sep=0pt] (X7) {\tiny{$X_7$}};
            \node[below= 0.1 of 37, inner sep=0pt] (X8) {\tiny{$X_8$}};

            \draw[<-, dashed, red] (11) -- (10);
            \draw[<-, dashed, red] (23) -- (22);
            \draw[<-, dashed, red] (22) -- (21);
            \draw[<-, dashed, red] (21) -- (20);
            \draw[<-, dashed, red, out=135] (23) to[out=150,in=30]  (21);
            \draw[<-, dashed, red, out=135] (22) to[out=-150,in=-30]  (20);





        \end{tikzpicture}
    \end{minipage}
    \caption{A balanced partition tree of a probabilistic neural circuit with eight variables. The partition tree (in black) describes how the variables decompose (in terms of the scope function). The edges in red indicate functional (neural) dependencies between partitions.}
    \label{fig:partition_graph}
\end{figure}

Ignoring for now neural dependencies in PNCs (\ie the red edges in the partition tree in Figure~\ref{fig:partition_graph}) we describe the layer-wise operations. Let us assume, for the sake of simplicity, that the number of variables $N$ is a power of $2$. For instance, the circuits in Figure~\ref{fig:circuit} has $N=4$ and the partition diagram in Figure~\ref{fig:partition_graph} has $N=8$. Given that we merge partitions pairwise at each layer via multiplication, we obtain for the number of layers in a circuit $N_L= \log_2 N+1$. Formally, we express the layer-wise product units as follows:
\begin{talign}
    \component_{l,p,1,c} = \component_{l-1,2(p-1)+1,2,c} \times \component_{l-1,2(p-1)+2,2,c},
    \tag{ProductLayer}
\end{talign}
which describes how to compute the value of a component given the components of the previous layer.
The meaning of the indices is described in Section~\ref{sec:prelim}. Using this notation we can also express the sum units at the leaves and at the root:
\begin{talign}
    \component_{1,p,2,c}
     & =
    \sum_{c'=1}^{N_D} \weight_{1,p,c,c'} \times \component_{1,p,1,c'}
    \tag{SumLeaf}
    \\
    \component_{N_L,1,2,1}
     & =
    \sum_{c'=1}^{N_C} \weight_{N_L,1,1,c'} \times \component_{N_L,1,1,c'}
    \tag{SumRoot}
\end{talign}
In the equations above, $N_D$ denotes the number of initial components in the leaves and $N_C$ is the number of components throughout the circuits.
For the circuit in Figure~\ref{fig:circuit} we have $N_D{=}2$ and $N_C{=}3$.
The weights $w_{l,p,c.c'}$ are real valued constants and normalized over the $c'$ dimension. Note that, in contrast to the formulation in Section~\ref{sec:cpc},
weights (and also computation units in general) are not identified by a single index (\eg $k$ in $\circuit_k(\cdot)$) but by four indices.
To make this distinction explicit we denote the computation units by $\kappa_{l,p,i,c}$ instead of $p_k(\cdot)$.
Observe also that the root layer has only a single component. Hence, the $1$ as the last index instead of $c$.

Next, we describe the neural sum units in layered PNCs. To gain some intuition, consider the partial computation graph in Figure~\ref{fig:neural_dependencies}, where we show in more detail, compared to Figures ~\ref{fig:circuit} and~\ref{fig:partition_graph}, the neural dependencies present at the sum units. Formally, we express the values of neural sum units as follows:
\begin{talign}
     & \component_{l,p,2,c}     \tag{NeuralSumLayer}
    \\
     & =
    \sum_{c'=1}^{N_C} \nweight_{l,p,c,c'}(\component_{l, p-\nu_{l,p}:p-1, 1, 1:N_C})
    \times
    \component_{l,p,1,c'}
    \nonumber
\end{talign}
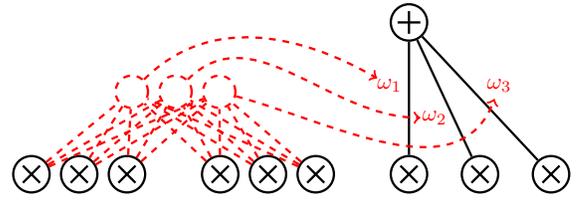
\begin{figure}[t]
    \begin{minipage}[c]{\linewidth}
        \centering
        \resizebox{0.9\columnwidth}{!}{
            \begin{tikzpicture}

                \prodnode[line width=\midlinewidth]{p21};
                \prodnode[line width=\midlinewidth, right=\halfdist of p21]{p22};
                \prodnode[line width=\midlinewidth, right=\halfdist of p22]{p23};

                \prodnode[line width=\midlinewidth, right=2\halfdist of p23]{p31};
                \prodnode[line width=\midlinewidth, right=\halfdist of p31]{p32};
                \prodnode[line width=\midlinewidth, right=\halfdist of p32]{p33};

                \node[draw, dashed, draw=red, circle, above right=0.875 and 0.35 of p23, inner sep=1pt,line width=\midlinewidth, minimum width=0.5cm] (w2){};
                \node[draw, dashed, draw=red,circle, left= \halfdist of w2, inner sep=1pt,line width=\midlinewidth,minimum width=0.5cm] (w1) {};
                \node[draw, dashed, draw=red,circle, right= \halfdist of w2, inner sep=1pt,line width=\midlinewidth,minimum width=0.5cm] (w3) {};

                \draw[-,red,dashed,line width=\midlinewidth] (p21) to (w1);
                \draw[-,red,dashed,line width=\midlinewidth] (p21) to (w2);
                \draw[-,red,dashed,line width=\midlinewidth] (p21) to (w3);
                \draw[-,red,dashed,line width=\midlinewidth] (p22) to (w1);
                \draw[-,red,dashed,line width=\midlinewidth] (p22) to (w2);
                \draw[-,red,dashed,line width=\midlinewidth] (p22) to (w3);
                \draw[-,red,dashed,line width=\midlinewidth] (p23) to (w1);
                \draw[-,red,dashed,line width=\midlinewidth] (p23) to (w2);
                \draw[-,red,dashed,line width=\midlinewidth] (p23) to (w3);

                \draw[-,red,dashed,line width=\midlinewidth] (p31) to (w1);
                \draw[-,red,dashed,line width=\midlinewidth] (p31) to (w2);
                \draw[-,red,dashed,line width=\midlinewidth] (p31) to (w3);
                \draw[-,red,dashed,line width=\midlinewidth] (p32) to (w1);
                \draw[-,red,dashed,line width=\midlinewidth] (p32) to (w2);
                \draw[-,red,dashed,line width=\midlinewidth] (p32) to (w3);
                \draw[-,red,dashed,line width=\midlinewidth] (p33) to (w1);
                \draw[-,red,dashed,line width=\midlinewidth] (p33) to (w2);
                \draw[-,red,dashed,line width=\midlinewidth] (p33) to (w3);

                \prodnode[line width=\midlinewidth, right=2\halfdist of p33]{p11};
                \prodnode[line width=\midlinewidth, right=0.5 of p11]{p12};
                \prodnode[line width=\midlinewidth, right=0.5 of p12]{p13};

                \sumnode[line width=\midlinewidth, above=1.75 of p11]{s13};

                \draw[-,line width=\midlinewidth] (p12) to (s13);
                \draw[-,line width=\midlinewidth] (p11) to (s13);
                \draw[-,line width=\midlinewidth] (p13) to (s13);

                \node[circle,inner sep=0pt,dashed,red] (lw1) at ($(p11)!0.5!(s13)+(-0.3,0.2)$) {$\nweight_1$};
                \node[circle,inner sep=0pt,dashed,red] (lw2) at ($(p12)!0.5!(s13)+(-0.15,-0.3)$) {$\nweight_2$};
                \node[circle,inner sep=0pt,dashed,red] (lw3) at ($(p13)!0.5!(s13)+(0.3,0.2)$) {$\nweight_3$};

                \draw[->,red,dashed,line width=\midlinewidth] (w1) to[out=45,in=150] (lw1);
                \draw[->,red,dashed,line width=\midlinewidth] (w2) to[out=45,in=-180] (lw2);
                \draw[->,red,dashed,line width=\midlinewidth] (w3) to[out=-15,in=-110] (lw3);

            \end{tikzpicture}
        }
    \end{minipage}
    \caption{Detailed graphical representation of neural dependencies in a PNC. The sum unit at the top outputs the weighted sum of the three product units at the bottom right. The weights for the sum are the outputs of a neural network for which it holds that $\sum_{i=1}^3=1$. They are computed using a neural network that takes as input the values of the six product units at the bottom left.}
    \label{fig:neural_dependencies}

\end{figure}

\begin{algorithm}[t]
    \caption{Layer-wise circuit evaluation}
    \label{alg:layer}
    \hspace*{\algorithmicindent} \textbf{Input:} $\xvars_o$. $\Xvars_m$ \\
    \hspace*{\algorithmicindent} \textbf{Output:} $p(\xvars_o)$ \\
    \hspace*{\algorithmicindent} \textbf{Require:} $\Xvars_o {\cup} \Xvars_m {=} \Xvars$,  $\Xvars_o {\sqsubset} \Xvars_m$
    \begin{algorithmic}[1]
        \State $\component_{p,c} \gets init(p,c,\xvars_o)$
        \State $\component_{p,c} \gets
            \sum_{c'=1}^{N_D} \weight_{1,p,c,c'} \times \component_{p,c'}$ \Comment{LeafLayer}
        \State $l \gets 2$
        \While {$l< N_L$}
        \State $\component_{p,c} \gets \component_{2p,c} \times \component_{2p+1,c} $ \Comment{ProductLayer}
        \State
        $ \component_{p,c}
            \gets
            \sum_{c'=1}^{N_C} \nweight_{l,p,c,c'}
            \times
            \component_{p,c'}
        $ \Comment{NeuralSumLayer}
        \State $l \gets l+1$
        \EndWhile
        \State $\component_{p,c} \gets \component_{2p,c} \times \component_{2p+1,c} $ \Comment{ProductLayer}
        \State $\component_{1,1} \gets \sum_{c'=1}^{N_C} \weight_{N_L,1,c,c'} \times \component_{1,c'}$ \Comment{RootSum}
        \State \Return $\component_{1,1}$
    \end{algorithmic}
\end{algorithm}

Here, $\nweight_{l,p,c,c'}(\cdot)$ denotes a neural network and $N_C$ denotes again the number of components. The input to the neural net are the values of the set of units $\component_{l, p-\nu_{l,p}{:}p-1, 1, 1{:}N_C}$. The notation $1{:}N_C$ denotes the range of components $[1,\dots,N_C]$ (including the first and last element). Similarly, $p{-}\nu_p{:}p{-}1$ denotes the range of partitions $[p{-}\nu_{l,p},\dots,p{-}1]$. The parameter $\nu_{l,p}$ is a hyperparameter and describes how many partitions \textit{`to the left'} are taken into consideration when computing the neural weights. For instance, in the partial computation graph in Figure~\ref{fig:neural_dependencies}, we have $\nu_{l,p}{=}2$. When setting $\nu_{l,p}{=}0$ we recover the special case of layered (non-conditional) probabilistic circuits.

Observe that it is those \textit{intra-layer} neural dependencies that induce an order on the random variables: we can only perform the computations of the units in the right branch of a partition tree if the values in the left branch are known. This holds recursively.
\begin{definition}
    Given a poset of random variables $\poset (\Xvars)$, we call a neural sum layer valid if the (partial) variable order it induces respects $\poset(\Xvars)$.
\end{definition}

Using the equations above to compute the leaf, product, sum, and root layers, we can also write down the pseudocode for marginal inference in layered PNCs, which we give in Algorithm~\ref{alg:layer}.
For ease of exposition we assume again that the number of variables $\Xvars$ is a power of two, and we refer to our implementation for the general case\footnote{\url{https://github.com/pedrozudo/ProbabilisticNeuralCircuits.git}}.

The algorithm takes as input a set of random variable instantiation $\xvars_o$ and a set of random variables $\Xvars_m$. The latter ought to be marginalized out, and the circuit evaluation computes the probability $p(\xvars_o)$.

The first step is to initialize the $\component_{p,c}$, which we do with the following function:
\begin{align}
    init(p,c,\xvars_o)
    =
    \begin{cases}
        f_{p.c}(\xvar) & \text{if $\xvar \in \xvars_o$  } \\
        1              & \text{otherwise}
    \end{cases}
\end{align}
Each combination of $p$ and $c$ corresponds to an index $k$ in Equation~\ref{eq:con:circuit}. Note that we forego in Algorithm~\ref{alg:layer} the possibility of introducing conditional dependencies in the leaves. The $1$'s in the second case result from marginalizing out the probability distributions in the leaves for the variables in $\Xvars_m$.

The algorithm then proceeds by first performing the computations in the leaf, before looping through the internal sum and product layers, and finishes with computing the root layer.
In contrast to the $4$-index notation introduced in Section~\ref{sec:prelim} we only use two indices here. This is because the identification of the layers happens implicitly using the $l$ counter of the while loop.

\subsection{Implementation Using Convolutions}

\begin{figure}[t]
    \begin{minipage}[c]{0.49\columnwidth}
        \centering

        \begin{tikzpicture}[scale=0.4]
            \draw[step=1cm, gray, thin] (0,0) grid (5,1);

            \fill[gray!50] (0,0) rectangle (1,1);
            \fill[gray!80] (1,0) rectangle (2,1);
            \path (0, -0.5) -- (0, -1.5);

        \end{tikzpicture}
    \end{minipage}
    \begin{minipage}[c]{0.49\columnwidth}
        \centering
        \begin{tikzpicture}[scale=0.4]
            \draw[step=1cm, gray, thin] (0,0) grid (3,3);

            \fill[gray!30] (0,2) rectangle (1,3);
            \fill[gray!50] (0 ,1) rectangle (1,2);
            \fill[gray!80] (1,2) rectangle (2,3);

        \end{tikzpicture}

    \end{minipage}

    \caption{Graphical representation of \textit{half kernels} used for neural sum layers in layered PNCs. On the left we see a kernel used for one-dimensional data while on the right we have a $3\times3$ kernel for two-dimensional data. The gray blocks indicate the learnable parameters of the \textit{half kernels}, while a white square indicates a parameter fixed to zero. Effectively, the convolutional layer is blind with regard to the inputs for these zero elements of the kernel.}
    \label{fig:kernel}

\end{figure}
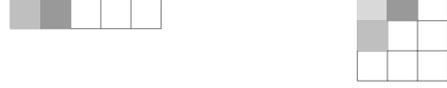

An efficient way of implementing neural sum layers is by means of convolutional neural networks.
We can readily see this if we interpret the components within a layer as the channels of the convolutional neural network and the partitions as the input dimension over which we perform the convolution.

Assuming that our one-dimensional data is ordered left to right, we can make sure that this order is respected throughout the neural sum layers by using convolutional networks with a \textit{half kernel} as depicted in the left of Figure~\ref{fig:kernel}. The convoluted channels can then be passed on to further layers. As a final activation function we use a softmax layer that normalizes the outputs such that they sum up to 1, per partition that is.
Importantly, the number of output channels has to match the number of components in the summation.

\section{Related Work}
\label{sec:related}

A first attempt at relaxing decomposability in probabilistic circuits was made by \citet{sharir2018sum} with the introduction of sum-product-quotient networks (SPQNs). SPQNs introduce the quotient unit (in addition to sum and products), which encode conditional probabilities within a circuit. We show that introducing extra units is unnecessary as the same effect can be obtained by generalizing sum and product units while retaining two types of computation units.

On a theoretical level this allows us to forego the introduction of expendable concepts such as conditional soundness~\cite[Definition 5]{sharir2018sum} or conditional and effective scopes of a unit~\citep[Section 3]{sharir2018sum}.

On the practical side we will show in Section~\ref{sec:experiments} that neural sum layers outperform the \textit{conditional mixing operator} (CMO) proposed as a building block for SPQNs~\citep[Definition 6]{sharir2018sum}. As a matter of fact, PNCs strictly generalize SPQNs constructed with CMOs.

\begin{corollary} CMO-SPQNs are PNCs.
    \label{cor:spq_worse}
\end{corollary}
\begin{proof}
    This can trivially be shown by picking $\nweight(\cdot)$ such that
    \begin{align}
         & \component_{l,p,2,c}     \tag{QuotientSumLayer}
        \\
         & =
        \frac{
            \sum_{c'=1}^{N_C} \weight_{l,p,c,c'}
            \times
            (\prod_{\component' \in  \component_{l, p-\nu_{l,p}:p-1, 1, c'}}\kappa' )
            \times
            \component_{l,p,1,c'}}
        {
            \sum_{c'=1}^{N_C}
            \weight_{l,p,c,c'}
            \times
            (\prod_{\component' \in  \component_{l, p-\nu_{l,p}:p-1, 1, c'}}\kappa' )
        }
        \nonumber
    \end{align}
    where we use our notation from Section~\ref{sec:pnc} to write down the CMO of~\citet{sharir2018sum}.
\end{proof}
Guided by Corollary~\ref{cor:spq_worse} we implemented CMO-SPQNs using a convolutional layer by simply fixing the non-zero elements of the kernel (\cf~Figure~\ref{fig:kernel}) to one. Performing the convolution on probabilities in log-space then simply corresponds to multiplying them in linear space. In this fashion we easily obtain the product present in the quotient sum layer.
An important difference between CMO-SPQNs and PNCs is that for the former components (or channels) must not mix with each other, which hinders their performance in terms of function approximation.

The limited expressive power of sum units was also noted by \citet{shao2022conditional}, which led them to introduce \textit{conditional sum-product networks}. The idea is to condition the weights of the sum units in a probabilistic circuit on the value of a random value (using neural networks). The main difference to our work is that these random variables live out-side of the circuits itself. Speaking in terms of Bayesian networks the approach of~\citet{shao2022conditional} is only capable of encoding a single Bayesian network while CPC encode hierarchical mixtures of Bayesian networks.

\section{Experimental Evaluation}
\label{sec:experiments}


For our experimental evaluation we used the MNIST family of dataset. That is, the original MNIST~\citep{deng2012mnist}, FashionMNIST~\citep{xiao2017fashion}, and also EMNIST~\citep{cohen2017emnist}.
We implemented PNCs (and also SPQNs) in PyTorch and Lightning\footnote{\url{https://lightning.ai/}}, and ran all our experiments on a DGX-2 machine with V100 Nvidia cards.

\subsection{How Do PNCs Fair Against PQCs and PCs?}
\label{sec:q1}

\begin{table*}[t]
    \footnotesize
    \centering
    \begin{tabular}{lccccccccc}
                          & PNC             & PQC    & PSC    & SHCLT  & HCLT   & CMC    & RAT-SPN & IDF     & BitSwap \\
        \cmidrule(lr){2-4}
        \cmidrule(lr){5-10}
        MNIST             & $\mathbf{0.87}$ & $1.20$ & $1.32$ & $1.14$ & $1.20$ & $1.28$ & $1.67$  & $1.90$  & $1.27$  \\
        FashionMNIST      & $\mathbf{2.51}$ & $3.47$ & $3.66$ & $3.27$ & $3.34$ & $3.55$ & $4.29$  & $3.47$  & $3.28$  \\
        EMNIST (mnist)    & $\mathbf{1.36}$ & $1.84$ & $2.07$ & $1.52$ & $1.77$ & --     & $2.56$  & $2.07$  & $1.88$  \\
        EMNIST (letters)  & $\mathbf{1.33}$ & $1.83$ & $2.07$ & $1.58$ & $1.80$ & --     & $2.73$  & $1.95$  & $1.84$  \\
        EMNIST (balanced) & $\mathbf{1.35}$ & $1.86$ & $2.16$ & $1.60$ & $1.82$ & --     & $2.78$  & $2.15$  & $1.96$  \\
        EMNIST (byclass)  & $\mathbf{1.27}$ & $1.76$ & $2.02$ & $1.54$ & $1.85$ & --     & $2.72$  & $1.98$  & $1.87$  \\
        \midrule
        {\#} parameters   & $2.8M$          & $2.6M$ & $2.6M$ & $7.0M$ & $7.0M$ & $0.1M$ & $7.0M$  & $24.1M$ & $2.8M$  \\
    \end{tabular}
    \caption{Test set bpd for MNIST datasets (lower is better). The last row shows the number of parameters for each model (the symbol $M$ stands for millions).}
    \label{tab:mnist}
\end{table*}

\paragraph*{Setup}
To answer the first question we trained a PNC, an SPQNs and a PC on the MNIST family of datasets and minimize the negative log-likelihood. In order to render the three models commensurable, all three have the same underlying circuit structure. That is, we start with a grid of $28\times 28$ pixels and merge, in an alternating fashion, rows and columns. This corresponds to product nodes in the circuits. Between each merge we perform a summation. For probabilistic (sum) circuits (PSCs) this is the usual sum unit, for probabilistic quotient circuits (PQCs or SPQNs) we use the conditional mixing operator, and for PNCs we use a neural sum layer. The kernel type used for PQCs and PNCs is the one depicted in the right of Figure~\ref{fig:kernel}.
For the leaves, which encode the $28\times 28$ pixels, we use one categorical distribution with $256$ categories for each pixel. This allows us to represent all possible pixel values.
Within the circuits we used $12$ components per partition.
Merging rows and columns in an alternating fashion and using the kernel from Figure~\ref{fig:kernel} induces a specific variable ordering on a 2-dimensional grid, which we graphically represent in Figure~\ref{fig:2dorder}.

All three models were trained for $100$ epochs using Adam~\citep{kingma2014adam} with a learning rate of $0.001$ and a batch size of $50$.
The best model was selected using a $90-10$ train-validation data split where we monitored the negative log-likelihood on the validation set. We refer to the configuration files of the experiments for more details.

\paragraph*{Results}
We compare the three architectures using \textit{bits per dimension}, which are calculated from the average negative log-likelihood ($\overline{NLL}$) as follows: $bpd = \nicefrac{\overline{NLL}}{(\log{2} \times D)}$, here $D=28^2$ for MNIST datasets.

The results are reported in Table~\ref{tab:mnist} in the first three columns. We see that quotient circuits outperform sum circuits when it comes to minimizing the negative log-likelihood (minimizing $bpd$). We also see that neural circuits, with their data dependent weights, outperform both other methods. Note that PNCs and PQCs allow for the same set of tractable queries, while PNCs are more performant.



\begin{figure}[t]




    \centering
    \begin{tikzpicture}[scale=0.6]
        \def\gridSize{4}
        \def\cellSize{1} 

        \draw[step=\cellSize,gray,very thin] (0,0) grid (\gridSize*\cellSize, \gridSize*\cellSize);

        \node at (0.5, 3.5) {1};
        \node at (1.5, 3.5) {2};
        \node at (2.5, 3.5) {5};
        \node at (3.5, 3.5) {6};

        \node at (0.5, 2.5) {3};
        \node at (1.5, 2.5) {4};
        \node at (2.5, 2.5) {7};
        \node at (3.5, 2.5) {8};

        \node at (0.5, 1.5) {9};
        \node at (1.5, 1.5) {10};
        \node at (2.5, 1.5) {13};
        \node at (3.5, 1.5) {14};

        \node at (0.5, 0.5) {11};
        \node at (1.5, 0.5) {12};
        \node at (2.5, 0.5) {15};
        \node at (3.5, 0.5) {16};
    \end{tikzpicture}


    \caption{For simplicity's sake, assume that we have a $4\times4$ grid of pixels (instead of the $28\times 28$ MNIST grid). Recursively merging rows and columns and using the kernel from Figure~\ref{fig:kernel} then gives us the  total pixel order as indicated in the grid. We can then marginalize out pixels with higher numbers before pixels with lower numbers.
        For the specific case here we could marginalize out the lower half of the $4\times 4$ image, while retaining a probability distribution for the upper half.
    }
    \label{fig:2dorder}

\end{figure}

\subsection{How Do PNCs Fair Against State of the Art?}

\paragraph*{Setup} We use again the same PNCs as in Section~\ref{sec:q1} and compare them to (decomposable) probabilistic circuits from the literature: hidden Chow-Liu trees (HCLT)~\citep{liu2021tractable}, sparse HCLT (SHCLT)~\citep{dang2022sparse}, random sum-product networks (RAT-SPN)~\citep{peharz2019random}, and continuous mixture circuits (CMC)~\citep{correia2023continuous}. For completeness, we also include the bpd for IDF (a flow-based approach)~\citep{hoogeboom2019integer} and for BitSwap (a hierarchical latent variable model)~\citep{kingma2019bit}. Note that the results reported for RAT-SPN were taken from~\citep{dang2022sparse}.

\paragraph*{Results}
We see again that PNCs outperform the other methods in terms of bpd. On the one hand this is due to the increased expressive power of PNCs obtained by relaxing decomposability (and thereby losing tractability). On the other hand, this is also due to the fact that PNCs use neural sum units. This can be seen by comparing the results for PQC and SHCLT. While PQC are in theory more expressive than SHCLT we don't see this in practice: the learned structure of SHCLTs overcomes this expressivity gap. However, we see that PNCs manage to outperform SHCLTs.
Interestingly our implementation of PCs (PSC) does produce lower bpd than RAT-SPNs, suggesting that the latter are a rather weak baseline when it comes to density estimation on image data.

\subsection{Can PNCs Perform Discriminative Learning?}
\paragraph*{Setup}
Assume that we have a data point $\xvars$ for the random variables in $\Xvars$ and a label $y$ for this data point. Using Bayes rule we can rewrite the discriminative probability using generative distributions:
$
    p(Y{=}y{\mid} \Xvars{=}\xvars )
    {=}
    \frac{
        p(x{\mid} y)
    }
    {
        \sum_{z\in \Omega(Y)} p(x{\mid} y)
    }
$,
where we assume that the class prior is identical for each of the classes belonging to the sample space $\Omega(Y)$. This means that for every class in $\Omega(Y)$ we have a separate distribution. That is, a separate circuit. In the case of MNIST and FashionMNIST we have ten classes. The resulting ten circuits are jointly optimized using cross-entropy~\citep{peharz2019random}.
Architecturally, we used again $12$ components per partition and the $10$ different circuits share all parameters but the parameters in the leaf layer and root layer.

Furthermore, instead representing pixels as categorical random variables (with sample space $\{0,\dots, 255 \}$) we model them as continuous random variables with samples belong to the interval $[0,1]$. We obtain this value $v$ by dividing the pixel by $255$. Each leaf has then two inputs: the corresponding value $v$ of of the pixel itself and $1-v$.
This follows the protocol of~\citep[]{liang2019learning}.
Apart from optimizing the cross-entropy instead of the log-likelihood the training protocol was identical to the one for density estimation.

\paragraph*{Results}

We report the comparison in terms of classification accuracy on the test set, which we show in Table~\ref{tab:discrim}. We compare PNCs and PQCs to logistic circuits (LCs)~\citep{liang2019learning} and RAT-SPNs~\Citep{peharz2019random}. We see that PNCs perform better than PQC. However, neither reaches the accuracies of LCs nor RAT-SPNs. We hypothesize that this is due to a lack of regularization. For instance, PNCs reach perfect train accuracy on MNIST and near perfect train accuracy on FashionMNIST.
Furthermore, the authors of LCs and RAT-SPNs reported having used aggressive regularization techniques -- for the former on their Github page\footnote{\url{https://github.com/UCLA-StarAI/LogisticCircuit}} and the latter in~\Citep[Section 4.2]{peharz2019random}.
While we experimented with various regularization techniques such as weight decay~\Citep{loshchilov2018fixing} or the stochastic delta rule~\citep{hanson1990stochastic}, we were not able to obtain consistent improvements. We leave the study of effective regularization techniques for discriminative learning with PNCs for future work.

\begin{table}[t]
    \footnotesize
    \centering
    \begin{tabular}{lcccc}
                     & PNC     & PQC     & LC     & RAT-SPN  \\
        \cmidrule(lr){2-3}
        \cmidrule(lr){4-5}
        MNIST        & $98.04$ & $97.38$ & $99.4$ & $98.29$  \\
        FashionMNIST & $88.84$ & $87.63$ & $91.3$ & $89.89 $ \\
    \end{tabular}
    \caption{Test accuracies for MNIST and FashionMNIST.}
    \label{tab:discrim}
\end{table}

\section{Conclusions \& Future Work}

We first introduced the concept of a conditional probabilistic circuit, from which we were then able to construct probabilistic neural circuits, which generalize probabilistic circuits and sum-product-quotient networks. Note that the construction of PNCs would not have been possible from the formulation of \citet{sharir2018sum}. Furthermore, our formulation allows us to intuitively interpret PNCs as neural approximations of deep mixtures of Bayesian networks.

Experimentally, we have shown that for density estimation PNCs deliver on the promise made by SPQNs. That is, giving up on tractability improves function approximation in practice. For the discriminative case the situation is more nuanced. While PNCs achieve perfect accuracy on the training set a lack of proper regularization techniques prevents them from matching accuracies obtained by competing methods. We would also like to note that more sophisticated architecture designs for PNCs could possibly further improve their performance. For instance, using different numbers of components per partition dramatically increased the performance of SHCLT when compared to HCTL.

In future work we would like to explore the potential of PNCs for sampling, a task that probabilistic models usually struggle with~\citep{lang2022elevating} as good likelihood estimates do not correlate with sample quality~\citep{theis2016note}.
This would also establish a tighter link to autoregressive models (ARM) (\eg PixelCNN~\citep{van2016conditional}) and we might use ideas developed there for PNCs. In this regard, any-order ARM~\citep{uria2014deep,shih2022training} seem to be of particular interest to PNCs as this could allow for arbitrary conditioning sets in PNCs.

Other open questions concern structure learning for PNCs and applying them to tabular data or finding applications of PNCs -- an obvious candidate would be lossless compression with circuits~\citep{liu2022lossless}, as any-order marginalization is not necessary.

\section*{Acknowledgements}

This project received funding from the Wallenberg AI, Autonomous Systems and Software Program (WASP) funded by the Knut and Alice Wallenberg-Foundation, as well as from the TAILOR Connectivity Fund (part of the TAILOR project funded by the EU Horizon 2020 research and innovation program under GA No 952215).
The author would like to thank the anonymous reviewers for the valuable comments.

\bibliography{references}

\begin{thebibliography}{37}
\providecommand{\natexlab}[1]{#1}

\bibitem[{Cohen et~al.(2017)Cohen, Afshar, Tapson, and
  Van~Schaik}]{cohen2017emnist}
Cohen, G.; Afshar, S.; Tapson, J.; and Van~Schaik, A. 2017.
\newblock EMNIST: Extending MNIST to handwritten letters.
\newblock In \emph{2017 international joint conference on neural networks}.

\bibitem[{Correia et~al.(2023)Correia, Gala, Quaeghebeur, de~Campos, and
  Peharz}]{correia2023continuous}
Correia, A.~H.; Gala, G.; Quaeghebeur, E.; de~Campos, C.; and Peharz, R. 2023.
\newblock Continuous mixtures of tractable probabilistic models.
\newblock In \emph{AAAI Conference on Artificial Intelligence}.

\bibitem[{Dang et~al.(2021)Dang, Khosravi, Liang, Vergari, and Van~den
  Broeck}]{dang2021juice}
Dang, M.; Khosravi, P.; Liang, Y.; Vergari, A.; and Van~den Broeck, G. 2021.
\newblock Juice: A julia package for logic and probabilistic circuits.
\newblock In \emph{AAAI Conference on Artificial Intelligence}.

\bibitem[{Dang, Liu, and Van~den Broeck(2022)}]{dang2022sparse}
Dang, M.; Liu, A.; and Van~den Broeck, G. 2022.
\newblock Sparse probabilistic circuits via pruning and growing.
\newblock In \emph{Advances in Neural Information Processing Systems}.

\bibitem[{Darwiche(2001)}]{darwiche2001decomposable}
Darwiche, A. 2001.
\newblock Decomposable negation normal form.
\newblock \emph{Journal of the ACM (JACM)}, 48(4): 608--647.

\bibitem[{Darwiche(2003)}]{darwiche2003differential}
Darwiche, A. 2003.
\newblock A differential approach to inference in Bayesian networks.
\newblock \emph{Journal of the ACM}, 50(3): 280--305.

\bibitem[{Darwiche(2011)}]{darwiche2011sdd}
Darwiche, A. 2011.
\newblock SDD: A new canonical representation of propositional knowledge bases.
\newblock In \emph{Twenty-Second International Joint Conference on Artificial
  Intelligence}.

\bibitem[{Delalleau and Bengio(2011)}]{delalleau2011shallow}
Delalleau, O.; and Bengio, Y. 2011.
\newblock Shallow vs. deep sum-product networks.
\newblock \emph{Advances in neural information processing systems}, 24.

\bibitem[{Deng(2012)}]{deng2012mnist}
Deng, L. 2012.
\newblock The mnist database of handwritten digit images for machine learning
  research.
\newblock \emph{IEEE signal processing magazine}.

\bibitem[{Di~Mauro et~al.(2021)Di~Mauro, Gala, Iannotta, and
  Basile}]{di2021random}
Di~Mauro, N.; Gala, G.; Iannotta, M.; and Basile, T.~M. 2021.
\newblock Random probabilistic circuits.
\newblock In \emph{Uncertainty in Artificial Intelligence}.

\bibitem[{Hanson(1990)}]{hanson1990stochastic}
Hanson, S.~J. 1990.
\newblock A stochastic version of the delta rule.
\newblock \emph{Physica D: Nonlinear Phenomena}, 42(1-3): 265--272.

\bibitem[{Hoogeboom et~al.(2019)Hoogeboom, Peters, Van Den~Berg, and
  Welling}]{hoogeboom2019integer}
Hoogeboom, E.; Peters, J.; Van Den~Berg, R.; and Welling, M. 2019.
\newblock Integer discrete flows and lossless compression.
\newblock In \emph{Advances in Neural Information Processing Systems}.

\bibitem[{Kileel, Trager, and Bruna(2019)}]{kileel2019expressive}
Kileel, J.; Trager, M.; and Bruna, J. 2019.
\newblock On the expressive power of deep polynomial neural networks.
\newblock \emph{Advances in neural information processing systems}, 32.

\bibitem[{Kingma and Ba(2014)}]{kingma2014adam}
Kingma, D.~P.; and Ba, J. 2014.
\newblock Adam: A method for stochastic optimization.
\newblock \emph{arXiv preprint arXiv:1412.6980}.

\bibitem[{Kingma, Abbeel, and Ho(2019)}]{kingma2019bit}
Kingma, F.; Abbeel, P.; and Ho, J. 2019.
\newblock Bit-swap: Recursive bits-back coding for lossless compression with
  hierarchical latent variables.
\newblock In \emph{International Conference on Machine Learning}.

\bibitem[{Kisa et~al.(2014)Kisa, Van~den Broeck, Choi, and
  Darwiche}]{kisa2014probabilistic}
Kisa, D.; Van~den Broeck, G.; Choi, A.; and Darwiche, A. 2014.
\newblock Probabilistic sentential decision diagrams.
\newblock In \emph{Fourteenth International Conference on the Principles of
  Knowledge Representation and Reasoning}.

\bibitem[{Lang et~al.(2022)Lang, Mundt, Ventola, Peharz, and
  Kersting}]{lang2022elevating}
Lang, S.; Mundt, M.; Ventola, F.; Peharz, R.; and Kersting, K. 2022.
\newblock Elevating perceptual sample quality in PCs through differentiable
  sampling.
\newblock In \emph{NeurIPS 2021 workshop on pre-registration in machine
  learning}. PMLR.

\bibitem[{Liang and Van~den Broeck(2019)}]{liang2019learning}
Liang, Y.; and Van~den Broeck, G. 2019.
\newblock Learning logistic circuits.
\newblock In \emph{AAAI Conference on Artificial Intelligence}.

\bibitem[{Liu, Mandt, and Van~den Broeck(2022)}]{liu2022lossless}
Liu, A.; Mandt, S.; and Van~den Broeck, G. 2022.
\newblock Lossless Compression with Probabilistic Circuits.
\newblock In \emph{International Conference on Learning Representations}.

\bibitem[{Liu and Van~den Broeck(2021)}]{liu2021tractable}
Liu, A.; and Van~den Broeck, G. 2021.
\newblock Tractable regularization of probabilistic circuits.
\newblock In \emph{Advances in Neural Information Processing Systems}.

\bibitem[{Loshchilov and Hutter(2018)}]{loshchilov2018fixing}
Loshchilov, I.; and Hutter, F. 2018.
\newblock Fixing weight decay regularization in adam.

\bibitem[{Martens and Medabalimi(2014)}]{martens2014expressive}
Martens, J.; and Medabalimi, V. 2014.
\newblock On the expressive efficiency of sum product networks.
\newblock \emph{arXiv preprint arXiv:1411.7717}.

\bibitem[{Peharz et~al.(2020)Peharz, Lang, Vergari, Stelzner, Molina, Trapp,
  Van~den Broeck, Kersting, and Ghahramani}]{peharz2020einsum}
Peharz, R.; Lang, S.; Vergari, A.; Stelzner, K.; Molina, A.; Trapp, M.; Van~den
  Broeck, G.; Kersting, K.; and Ghahramani, Z. 2020.
\newblock Einsum networks: Fast and scalable learning of tractable
  probabilistic circuits.
\newblock In \emph{International Conference on Machine Learning}.

\bibitem[{Peharz et~al.(2015)Peharz, Tschiatschek, Pernkopf, and
  Domingos}]{peharz2015theoretical}
Peharz, R.; Tschiatschek, S.; Pernkopf, F.; and Domingos, P. 2015.
\newblock On theoretical properties of sum-product networks.
\newblock In \emph{Artificial Intelligence and Statistics}.

\bibitem[{Peharz et~al.(2019)Peharz, Vergari, Stelzner, Molina, Shao, Trapp,
  Kersting, and Ghahramani}]{peharz2019random}
Peharz, R.; Vergari, A.; Stelzner, K.; Molina, A.; Shao, X.; Trapp, M.;
  Kersting, K.; and Ghahramani, Z. 2019.
\newblock Random sum-product networks: A simple and effective approach to
  probabilistic deep learning.
\newblock In \emph{Uncertainty in Artificial Intelligence}.

\bibitem[{Poon and Domingos(2011)}]{poon2011sum}
Poon, H.; and Domingos, P. 2011.
\newblock Sum-product networks: A new deep architecture.
\newblock In \emph{2011 IEEE International Conference on Computer Vision
  Workshops}. IEEE.

\bibitem[{Shao et~al.(2022)Shao, Molina, Vergari, Stelzner, Peharz, Liebig, and
  Kersting}]{shao2022conditional}
Shao, X.; Molina, A.; Vergari, A.; Stelzner, K.; Peharz, R.; Liebig, T.; and
  Kersting, K. 2022.
\newblock Conditional sum-product networks: Modular probabilistic circuits via
  gate functions.
\newblock \emph{International Journal of Approximate Reasoning}, 140: 298--313.

\bibitem[{Sharir and Shashua(2018)}]{sharir2018sum}
Sharir, O.; and Shashua, A. 2018.
\newblock Sum-product-quotient networks.
\newblock In \emph{International Conference on Artificial Intelligence and
  Statistics}.

\bibitem[{Shih, Sadigh, and Ermon(2021)}]{shih2021hyperspns}
Shih, A.; Sadigh, D.; and Ermon, S. 2021.
\newblock Hyperspns: Compact and expressive probabilistic circuits.
\newblock \emph{Advances in Neural Information Processing Systems}.

\bibitem[{Shih, Sadigh, and Ermon(2022)}]{shih2022training}
Shih, A.; Sadigh, D.; and Ermon, S. 2022.
\newblock Training and Inference on Any-Order Autoregressive Models the Right
  Way.
\newblock \emph{Advances in Neural Information Processing Systems}.

\bibitem[{Theis, van~den Oord, and Bethge(2016)}]{theis2016note}
Theis, L.; van~den Oord, A.; and Bethge, M. 2016.
\newblock A note on the evaluation of generative models.
\newblock In \emph{International Conference on Learning Representations}.

\bibitem[{Uria, Murray, and Larochelle(2014)}]{uria2014deep}
Uria, B.; Murray, I.; and Larochelle, H. 2014.
\newblock A deep and tractable density estimator.
\newblock In \emph{International Conference on Machine Learning}.

\bibitem[{Van~den Oord et~al.(2016)Van~den Oord, Kalchbrenner, Espeholt,
  Vinyals, Graves et~al.}]{van2016conditional}
Van~den Oord, A.; Kalchbrenner, N.; Espeholt, L.; Vinyals, O.; Graves, A.;
  et~al. 2016.
\newblock Conditional image generation with pixelcnn decoders.
\newblock \emph{Advances in neural information processing systems}, 29.

\bibitem[{Vergari et~al.(2021)Vergari, Choi, Liu, Teso, and Van~den
  Broeck}]{vergari2021compositional}
Vergari, A.; Choi, Y.; Liu, A.; Teso, S.; and Van~den Broeck, G. 2021.
\newblock A compositional atlas of tractable circuit operations for
  probabilistic inference.
\newblock In \emph{Advances in Neural Information Processing Systems}.

\bibitem[{Xiao, Rasul, and Vollgraf(2017)}]{xiao2017fashion}
Xiao, H.; Rasul, K.; and Vollgraf, R. 2017.
\newblock Fashion-mnist: a novel image dataset for benchmarking machine
  learning algorithms.
\newblock \emph{arXiv preprint arXiv:1708.07747}.

\bibitem[{Zhang, Juba, and Van~den Broeck(2021)}]{zhang2021probabilistic}
Zhang, H.; Juba, B.; and Van~den Broeck, G. 2021.
\newblock Probabilistic generating circuits.
\newblock In \emph{International Conference on Machine Learning}.

\bibitem[{Zhao, Poupart, and Gordon(2016)}]{zhao2016unified}
Zhao, H.; Poupart, P.; and Gordon, G.~J. 2016.
\newblock A unified approach for learning the parameters of sum-product
  networks.
\newblock \emph{Advances in neural information processing systems}.

\end{thebibliography}

\end{document}